\documentclass[sigconf]{acmart}

\AtBeginDocument{%
  \providecommand\BibTeX{{%
    \normalfont B\kern-0.5em{\scshape i\kern-0.25em b}\kern-0.8em\TeX}}}

\copyrightyear{2020}
\acmYear{2020}
\setcopyright{acmlicensed}
\acmConference[Preprint]{Preprint}
\acmBooktitle{Preprint}
\acmPrice{}
\acmDOI{}
\acmISBN{}

\usepackage[utf8]{inputenc} % allow utf-8 input
\usepackage[T1]{fontenc}      % use 8-bit T1 fonts
\usepackage{hyperref}           % hyperlinks
\usepackage{url}                     % simple URL typesetting
\usepackage{booktabs}         % professional-quality tables
\usepackage{amsfonts}         % blackboard math symbols
\usepackage{nicefrac}           % compact symbols for 1/2, etc.
\usepackage{microtype}        % microtypography
\usepackage{subcaption}
\usepackage{lipsum}
\usepackage{comment}
\usepackage{multirow}
\usepackage{nth}

\usepackage[utf8]{inputenc}
\usepackage{amsmath,amssymb,amsthm,dsfont,graphicx,xcolor,bm}
\usepackage[capitalize,noabbrev]{cleveref}

\usepackage{algorithm}
\usepackage{algorithmicx}
\usepackage[noend]{algpseudocode}

\usepackage[backgroundcolor=White,textwidth=0.8in]{todonotes}
\newtheorem{proposition}{Proposition}

\renewcommand{\qed}{\rule{5pt}{5pt}}

\newcommand{\cB}{\mathcal{B}}

\newcommand{\cA}{\mathcal{A}}

\DeclareMathOperator*{\argmax}{arg\,max\,}
\DeclareMathOperator*{\argmin}{arg\,min\,}
\mathchardef\mhyphen="2D

\makeatletter
\DeclareRobustCommand\onedot{\futurelet\@let@token\@onedot}
\def\@onedot{\ifx\@let@token.\else.\null\fi\xspace}

\makeatother

\newcommand{\LinUCBPBMRank}{{\tt LinUCBPBMRank}}
\newcommand{\LinTSPBMRank}{{\tt LinTSPBMRank}}

\title[Learning to Rank in the Position Based Model with Bandit Feedback]{Learning to Rank in the Position Based Model \texorpdfstring{\\}{} with Bandit Feedback}

\author{Ermis Beyza, Patrick Ernst, Yannik Stein, Giovanni Zappella}
\email{{ermibeyz, peernst, syannik, zappella}@amazon.de}
\affiliation{
   \institution{Amazon}
   \city{Berlin}
   \country{Germany}
}

\begin{document}

\begin{abstract}
Personalization is a crucial aspect of many online experiences. In particular, content ranking is often a key component in delivering sophisticated personalization results.
Commonly, supervised learning-to-rank methods are applied, which suffer from bias introduced
during data collection by production systems in charge of producing the ranking.
To compensate for this problem, we leverage \textit{contextual multi-armed bandits}.
We propose novel extensions of two well-known algorithms viz. LinUCB
and Linear Thompson Sampling to the ranking use-case.
To account for the biases in a production environment, we employ the \textit{position-based click model}.
Finally, we show the validity of the proposed algorithms by conducting extensive offline
experiments on synthetic datasets as well as customer facing online A/B experiments.
\end{abstract}

\maketitle

\section{Introduction}
\label{sec:intro}
% !TEX root = paper.tex
The content catalogue in many online experiences today is too large to be disseminated by regular customers.
To explore and consume these catalogues, content providers
often present a selected subset of their content which is personalized
for easier consumption.
For example, almost all major music streaming services rely on vertical tile
interfaces, where the user interface is subdivided into rectangular blocks,
vertically and horizontally. The content of every tile is a
graphical banner. Usually, customers observe a limited number of tiles,
that sometimes even rotate every few seconds, where only one large banner is visible at each point in time.\\
The selected tiles displayed to the customer significantly impact the engagement with the service.
Moreover, the order in which
they are presented by the application strongly impacts their chance of being observed by the customer.
This clearly calls for the need to consider the order as well as the bias introduced by the visualization mechanism.
Generally, the selection and ranking of content are core operations in most modern recommendation and personalization systems.
In this problem setting, we need to leverage all available information to improve the customer experience.\\
\textbf{Related Work}.
Learning-to-rank approaches have been studied in practical settings (e.g., see \cite{zalando})
and there is additional work to address the presence of incomplete feedback (also known as ``bandit'' feedback)
(e.g., \cite{multiplayclaire, kveton2015cascading, wen2015efficient, komiyama2017position}).
Learning-to-rank can be cast as a combinatorial learning problem where, given a set of actions,
the learner has to select the ordered subset maximizing its reward.
A standard combinatorial problem with bandit feedback (e.g., see
\cite{cesa2012combinatorial, combes2015combinatorial}) would provide a single
feedback (e.g., click/no-click signal) for each subset of selected actions or
tiles, making the problem unnecessarily difficult.
A more benign formulation is to look at the problem as a semi-bandit problem, where the learner can observe feedback for each action,
eventually transformed by a function of the actions position in the ranking.
Recently, several relevant methods have been proposed for this kind of problem:
non-contextual bandit methods such as \cite{multiplayclaire, luedtke2016asymptotically, komiyama2017position, lattimore2018toprank}
\emph{do not leverage side-information} about customers or content and thus do not present a viable solution for our problem setting.
Different approaches offer solutions using complex click models (i.e., the cascade model \cite{kveton2015cascading, zong2016cascading}),
which can be effective on applications like re-ranking of search results, but are complex to extend to consider other aspects like additional
elements on the page since in practice they are often controlled by different subsystems.

The approaches described in \cite{Diaz:2018, Lalmas:2019, Carterette:2019} share the same problem space as this work,
but target different aspects of the problem, such as fairness, reward models, and evaluations.
\\
\textbf{Contribution}.
The first contribution of this paper is two different contextual linear bandit methods
for the so called \textit{Position-Based Model}~(PBM)~\cite{chuklin2015click},
which are straightforward to implement, maintain and debug.
Second, we provide an empirical study on techniques to estimate the position bias during the learning process.

Specifically, we introduce new algorithms derived from LinUCB and
linear Thompson Sampling with Gaussian posterior, addressing the problem of
learning with bandit feedback in the PBM.
This model assumes that the probability of the customer interacting with a piece of content is a function of the relevance of that
content and the probability that the customer will actually inspect that content allowing the model to be used in various scenarios.
To the best of our knowledge, this is the first contextual bandit approach using PBM.
Finally, we show the validity and versatility of our approach by conducting extensive
experiments on experiments on synthetic datasets as well as customer facing online A/B experiments,
including lessons learned with anecdotal results.

\section{Problem Setup}
\label{sec:setup}
% !TEX root = paper.tex
In the following we introduce the Position-Based Model (PBM) to distinguish rewards
for different ranking positions and afterwards the linear reward learning model.

\textbf{Position-Based Model}.
\label{sec:PBM}
PBM~\cite{craswell2008experimental,richardson2007predicting} is a click model where the probability of getting a click on an action
depends on both its relevance and position.
In this setting, each position is randomly observed by the user with some probability.
It is parameterized by both $L$ action dependent relevance scores, expressing the probability that an action is judged as relevant, and $L$ position dependent examination
probabilities $q \in \left[ 0,1 \right]^L$, where
$q_\ell$ denotes the examination probability that position $\ell$ was observed (also known as position bias).
The core assumption of PBM is that the events of an item being relevant and being observed are independent,
i.e. the probability of getting a click $C$ on action $a$ in position $\ell$ is:
$P(C=1 | x, a, \ell) = P(E=1 | \ell) P(R=1 | x, a)$.
% The observed click Bernoulli variable $C$ depends on two other
% hidden Bernoulli variables $E$ and $R$ where $R$ represents the relevance of an action $a$ to a context $x$ and $E$ represents the event whether a user examines an action $a$ at a certain position $\ell$.
Regardless of the items that are placed there: $q_\ell = P(E=1 | \ell) = P(C=1 | \ell)$,
we need to provide these parameters.
In Section~\ref{sec:pb_estimation}, we discuss how we derive these parameters.

\textbf{The Learning Model}.
\label{sec:learning_model}
We consider a linear bandit setting in which the taken action at each round is a
list of $L$ actions chosen from a given set $\{a_1, \ldots, a_K\}$ of size $K$. Accordingly, assuming a semi-bandit feedback, we receive a reward in the form of a list of feedbacks corresponding to each position of the recommended list.
%At each round $t$ of the learning process, a user lands on the page. It is described by user features $x_t$. The first step is to build a \emph{contextualized action set} $\cA_t=(\psi(a_i,x_t))_{i\in[K]}$ that is the set of vector created by applying some non-linear function to the user features and action vectors. The function $\psi$ here can be arbitrary.
At each round $t$ of the learning process, we obtain $K$ vectors in
$\mathds{R}^d$ that represent the available actions for the learner. We denote
these by $\cA_t=\{a_t^1,\ldots, a_t^K\}$ and the action
list selected at time $t$ will be denoted as $A_t = (A_t^1,\ldots,A_t^L)$,
where $A_t$ is a permutation of $\cA_t$.

The PBM is characterized by examination parameters $(q_\ell)_{1 \leq \ell \leq L}$, where $q_\ell$ is the probability that the user effectively observes the item in
position $\ell$. At round $t$, the selection $A_t$ is shown and the learner observes the complete feedback. However, the observation $Z_t^\ell$ at position $\ell$
is censored being the product of the examination variable $Y_t^\ell$ and the actual user feedback $C_t^\ell$ where $Y_t^\ell \sim \cB(q_\ell)$ and $C_t^\ell = {A_t^\ell}^T \theta + \eta_t^\ell$ with all $\eta_t^\ell$ being 1-subgaussian independent random variables. When the user considered the item in position $\ell$, $Y_t^\ell$ is unknown
to the learner and $C_t^\ell$ is the reward of the item shown in position $\ell$. Then, we can compute the expected payoff of each action in each position, conditionally
on the action:
$\mathds{E}[Z_t^\ell| A_t^\ell] = q_\ell {A_t^\ell}^T \theta$,
where $\theta \in \mathds{R}^d$ is the unknown model parameter. At each step $t$, the learner is asked to make a list of $L$ actions $A_t$ that may depend on
the history of observations and actions taken. As a consequence to this choice, the learner is rewarded with $r_{A_t}=\sum_{\ell=1}^L Z_t^\ell$, where $Z_t=(Z_t^1,\ldots,Z_t^L)=(C_t^1 Y_t^1,\ldots,C_t^L Y_t^L)$. The goal of the learner is to maximize the total reward $\sum_{t=1}^T r_{A_t}$ accumulated over the course of $T$
rounds.
%In other words, we evaluate the performance of a learner by its expected cumulative regret:
%\[
%R(T) = \mathds{E} \left[ \sum_{t=1}^T R_{A_t} \right]
%\]
%where $R_{A_t}=r_{A_t^*}-r_{A_t}$ is the instantaneous regret of the learner at time $t$ and $A_t^*=\argmax_{A_t \in \cA_t}r_{A_t}$ is the optimal list of items that maximizes the expected reward.
% It is well known \cite{auer2002finite} that this type of problem requires to solve an \emph{exploration-exploitation} trade-off. The model parameter $\theta$ is unknown. However, using its past observation the learner can compute an unbiased estimator $\hat{\theta}_t$ of $\theta$.

\section{Ranking Algorithms}
\label{sec:ranking}
% !TEX root = paper.tex
\begin{figure*}
\algrenewcommand{\algorithmiccomment}[1]{\hskip0.5em$\rightarrow$ #1}
\resizebox{0.49\textwidth}{!}{
\begin{minipage}[t]{0.49\textwidth}
  \begin{algorithm}[H]
  \caption{$\LinUCBPBMRank$}
  \label{alg:LinRankUCB}
  \begin{algorithmic}
    \State {\bf Input:} Position Bias Parameters $(q_1,\ldots, q_L)$,
    \\confidence level $\delta>0$,
    regularization $\lambda$.
  % \State {\bf Initialization:} Pick  $(A_1^1, \ldots, A_1^\ell) \in \cA_1$ uniformly\\at random.
  % \State{Receive a list feedback $(Y_1^1 \ldots, Y_1^L)$}
  % \State $V_1 = \sum_\ell q_\ell^2 A_1^\ell {A_1^\ell}^T $, $b_1 = \sum_\ell q_\ell Y_1^\ell A_1^\ell$
  %\Comment Weighting based on position
  \For{$t=1,\ldots,T$}
    \State Get the contextualized actions $\cA_t$,
    \State Compute $\hat{\theta}_t$ as in Prop.~\ref{prop1} and for all $a \in \cA_t$,
    \[
    U_t(a) = a^T \hat{\theta}_t + \sqrt{f_{t,\delta} \Vert a \Vert_{V^{{-1}}}^2}
    \]
    \State Build Top-L action list
    \[
      A_t \in \argmax_{a \in \cA_t} \sum_\ell q_\ell U_t(a)
    \]
    \State (ties broken arbitrarily)
    \State Update $V_t \gets V_{t-1} + \sum_\ell q_\ell^2 A_t^\ell {A_t^\ell}^T$
    \State Receive feedback for round $t$
    \State Update $b_t \gets b_{t-1} + \sum_\ell q_\ell Y_t^\ell A_t^\ell$
    \vspace{0.2cm}
  \EndFor
  \end{algorithmic}
  \end{algorithm}
\end{minipage}
}
% \hspace{-5mm}
\resizebox{0.49\textwidth}{!}{
\begin{minipage}[t]{0.49\textwidth}
  \begin{algorithm}[H]
  \caption{$\LinTSPBMRank$}
  \label{alg:LinRankTS}
  \begin{algorithmic}
  \State{{\bf Input:} Position Bias Parameters $(q_1,\ldots, q_L)$,
    confidence\\level $\delta>0$, prior precision parameters $\alpha_0$ and $\beta_0$, so that\\$\sigma^2 \sim \mathcal{IG}(\alpha_0, \beta_0)$ and $p_0(\theta)=\mathcal{N}(0,\sigma I)$.}
  \For{$t=1,\ldots,T$}
    \State Get the contextualized actions $\cA_t$,
    \State Sample $\tilde{\theta}_t \sim p_{t-1}$
    \State Compute scores for all
    \[
    a \in \cA_t: s_t(a) = a^T \tilde{\theta}_t
    \]
    \State Build Top-L action list,
    \[
      A_t \in \argmax_{a \in \cA_t} \sum_\ell q_\ell s_t(a)
    \]
    \State Update $V_t \gets V_{t-1} + \sum_\ell q^2_\ell A_t^\ell {A_t^\ell}^T$
  % \Comment Weighting based on position
    \State Receive feedback for round $t$
    \State Update $b_t \gets b_{t-1} + \sum_\ell q_\ell Y_t^\ell A_t^\ell$
    \vspace{-0.1cm}
  \EndFor
  \end{algorithmic}
  \end{algorithm}
\end{minipage}
}
\end{figure*}
We now introduce two contextual bandit algorithms for learning to rank in the PBM. The first one is named \emph{LinUCB-PBMRank} that is a variation of LinUCB~\cite{abbasi2011improved,chu2011contextual,dudik2011efficient}, the contextual version of the optimistic approaches inspired by UCB1. The second algorithm, called \emph{LinTS-PBMRank},
is Bayesian approach to exploration and it is a variation of linear Thompson Sampling (LinTS)~\cite{agrawal2013thompson}.

\subsection{The Optimistic Approach: LinUCB-PBMRank}.\label{sec:linucbrank}
The LinUCB algorithm for contextual bandit problem for a single action case at each time $t$, obtains a least square estimator for $\theta$
using all past observations:
%\begin{align*}
$  \hat{\theta}_t = \argmin_{\tilde{\theta} \in \mathbb{R}^d} \sum_{s=1}^{t-1} ( C_s - {A_s}^T \tilde{\theta})^2 + \lambda \|\tilde{\theta} \|^2 $.\\
%\end{align*}
%where $A_s$ is a contextualized action, $C_s = A_s^T \theta + \eta_t$ and $\mathds{E}[C_s | A_s] = {A_s}^T \theta$.
%
We can now derive a conditionally unbiased estimator of the model parameter $\theta$ for the ranking case in the PBM as a least square solution of
%\begin{equation}
    \label{eq:theta_hat}
 $   \hat{\theta}_t = \argmin \sum_{s=1}^{t-1} \sum_{\ell=1}^L ( Z_s^\ell - q_\ell {A_s^\ell}^T \tilde{\theta})^2 + \lambda \|\tilde{\theta} \|^2 $.
%\end{equation}
%
%
\begin{proposition}\label{prop1}
The solution to the convex optimization problem formulated above gives a closed form solution for the estimator $\hat{\theta}$:
\small
\begin{equation}
\hat{\theta}_t = V_t^{-1} b_t = \left( \sum_{\ell=1}^L q_\ell^2 V_t^\ell + \lambda I \right)^{-1} \left( \sum_{\ell=1}^L q_\ell b_t^\ell \right)
\label{eqn:linucb_theta}
\end{equation}
\normalsize
where
$\forall\, \ell \in [L], V_t^\ell = \sum_{s=1}^{t-1} A_s^\ell {A_s^\ell}^T \quad \text{and} \quad b_t^\ell = \sum_{s=1}^{t-1} Z_s^\ell A_s^\ell$.
\label{prop:prop1}
\end{proposition}

\begin{proof}
Computing the gradient of the cost function is Eq~\ref{eqn:linucb_theta} and solving the equation leads to
\begin{align*}
    \sum_{s=1}^{t-1} 2 \sum_{\ell=1}^L q_\ell A_s^\ell Z_s^\ell \left( Z_s^\ell - q_\ell {A_s^\ell}^T \theta \right) + \lambda \theta & = 0 \\
    \sum_{\ell=1}^L q_\ell \sum_{s=1}^{t-1} A_s^\ell Z_s^\ell -
    \sum_{\ell=1}^{L} q_\ell^2 \sum_{s=1}^{t-1} Z_s^\ell A_s^\ell ({A_s^\ell}^T \theta + \lambda \theta  &=0
\end{align*}
that produce the stated solution.
\end{proof}

The pseudocode of LinUCB for ranking in the PBM is given in Algorithm~\ref{alg:LinRankUCB}.
%A good practice in general is to set the regularization $\lambda$ to a reasonably small value as compared to the expected value of the rewards, \emph{i.e} with respect to the norm of the parameters and actions. The larger $\lambda$, the more $\LinUCBPBMRank$ is going to explore at the beginning of the learning. The step that consists in build the action list that maximizes the cost function is simply done by sorting the UCBs in decreasing order and pulling the list of the Top-L best actions. The pseudocode of LinUCB for ranking in the PBM is given in Algorithm~\ref{alg:LinRankUCB}.

\subsection{The Bayesian Approach: LinTS-PBMRank}.\label{sec:lintsrank}
From a Bayesian point of view, the problem can be formulated as a posterior estimation of the parameter $\theta$. Here, the true observations $Z_t^\ell$ is
replaced by its conditional expectation given the censored position variables $Y_t^\ell \sim \mathcal{B}(q_\ell)$. We introduce the filtration $ \mathcal{F}_t$ as
the union of history until time $t-1$, and the contexts at time $t$, $\mathcal{F}_t=(A_1,Z_1,\ldots,A_t)$ such that for all $t, \, \ell$, $\mathds{E}[Z_t^\ell | \mathcal{F}_t] = \mathds{E}[C_t^\ell | \mathcal{F}_t] \mathds{E}[Y_t^\ell | \mathcal{F}_t] = q_\ell ({A_t^\ell}^T\theta)$.
We present a fully Bayesian treatment of Linear  Thompson Sampling where we assume $\sigma^2$ follows an Inverse-Gamma distribution and $\theta$ follows
a multivariate Gaussian:
\begin{align*}
\sigma^2  \sim \mathcal{IG}(\alpha_0, \beta_0) := p_0(\sigma^2) \\
\theta  \sim \mathcal{N}(\theta_0, \sigma^2 V_0^{-1}):=p_0(\theta)  \\
Z_t^{\ell} | A_t^{\ell}, \theta, q_\ell, \sigma^2  \sim \mathcal{N}(q_\ell \theta^T A_t^{\ell}, \sigma^2)
\end{align*}
For the above model, the joint model posterior $p(\theta, \sigma^2 | \mathcal{F}_t)$ follows a Normal-Inverse-Gamma distribution.
We can compute the posterior of the full-Bayesian approach as follows:
\begin{align*}
p(\tilde{\theta} | \mathcal{F}_t) & \propto p_0(\sigma^2) p_0(\theta) \prod_{t=1}^T \prod_{\ell=1}^L p(Z_t^\ell | \theta_t, \mathcal{F}_t) \\
& \propto \exp \lbrace - \frac{1}{2 \sigma^2} \sum_{\ell=1}^L (Z_t^\ell - q_\ell {A_t^\ell}^T \theta_t)^T (Z_t^\ell - q_\ell {A_t^\ell}^T \theta_t) \rbrace  \\
& \exp \lbrace - \frac{1}{2 \sigma^2} (\theta_t - \theta_0)^T V_0^{-1} (\theta_t - \theta_0) \rbrace \lbrace (\sigma^2)^{-(\alpha_0+1)} \exp \left(-\frac{\beta_0}{\sigma^2}\right) \rbrace
\end{align*}
We rearrange the posterior to formalize the posterior mean $\theta_t$ and the variance $V_t^{-1}$ in closed form. First, we rewrite the quadratic terms
in the exponential as a quadratic form:
\begin{align*}
Q(\tilde{\theta}, \sigma^2) & = (Z_t^\ell - q_\ell {A_t^\ell}^T \theta_t)^T (Z_t^\ell - q_\ell {A_t^\ell}^T \theta_t) + (\theta_t - \theta_0)^T V_0^{-1} (\theta_t - \theta_0) \\
& = (\tilde{Z}_t^\ell - W \theta_t)^T (\tilde{Z}_t^\ell - W \theta_t)
\end{align*}
where
\begin{align*}
\tilde{Z}_t^\ell =  \left( \begin{matrix}  Z_t^\ell \\ V_0^{\frac{1}{2}} \theta_0 \end{matrix} \right) \qquad   \text{and}  \qquad
W = \left( \begin{matrix} q_\ell A_t^\ell \\ V_0^{\frac{1}{2}} \end{matrix} \right)
\end{align*}
In this case $V_t^{-1} = (W^T W)^{-1} = (q_\ell^2 {A_t^\ell}^T A_t^\ell + V_0^{-1})^{-1}$ and $\theta_t = \Sigma_t (W^T \tilde{Z}_t^\ell)
= V_t^{-1} (q_\ell {A_t^\ell}^T \tilde{Z}_t^\ell + V_0^{-1} \theta_0)$.
At each time $t$, we sample one vector from the posterior for each action to compute the scores. The parameters of this posterior in terms
of the parameters at time $t-1$ are analytically computed as:
\begin{equation*}
\begin{aligned}[c]
V_t &= (\sum_t q_\ell^2 A_t^\ell {A_t^\ell}^T + V_0)\\ %= \Sigma_{t-1} - \frac{z_t z_t^T}{1 + z_t^T A_t^\ell}  \\
\alpha_t &= \alpha_0 + \frac{t}{2}
\end{aligned}
\qquad \qquad
\begin{aligned}[c]
\theta_t &= V_t^{-1} b_t   \\
\beta_t &= \beta_0 + \frac{1}{2} (\eta_t - \theta_t^t b_t)
\end{aligned}
\end{equation*}
where
\begin{equation*}
\begin{aligned}[c]
%M_t &= M_{t-1} +  \\
V_0 &= \lambda I\\
b_t &= b_{t-1} + q_\ell Z_t^\ell A_t^\ell
\end{aligned}
\qquad \qquad \qquad \qquad
\begin{aligned}[c]
\eta_t &= \eta_{t-1} + \sum_{\ell}{Z_t^\ell}^2 \\
z_t &= V_{t-1}^{-1} q_\ell A_t^\ell
\end{aligned}
\end{equation*}

We can simply apply the Sherman-Morrison identity~\cite{sherman1950adjustment} that computes the inverse of the sum of an invertible matrix as the outer
product of vectors to improve computational efficiency. The linear Thompson Sampling to rank (LinUCB-PBMRank) is summarized in
Algorithm~\ref{alg:LinRankTS}.
For dense action vectors the above update schema is computed in $\mathcal{O}(d^2)$.

\section{Position Bias Estimation}
\label{sec:pb_estimation}
% !TEX root = paper.tex
Accurate estimation of the position bias is crucial for unbiased learning-to-rank from implicit click data. We can provide these parameters either as fixed
or use an automatic parameter estimation method.
Using fixed hyperparameters in a production environment with many different use-cases
and continuously expanding use cases can be quite challenging
in terms of maintenance and scaling.
To avoid that, we evaluate three automatic estimation methods: \textit{i)} estimate using the click-through
rate (CTR) per position by updating them online after observing each record \textit{ii)} a supervised learning approach leveraging the Bayesian Probit
regression (PR) model and \textit{iii)} bias estimation using an expectation-maximization (EM) algorithm.\\

\subsection{CTR per position}.
% \todog{can we re-run experiments with the normalized CTR per position?}
One of the most commonly used quantities in click log studies is click-through rates (CTR) at different positions~\cite{chuklin2015click,joachims2017accurately}.
A common heuristic used in these cases is the rank-based CTR model where the click probability depends on the rank of the document
%\small
%\begin{equation}
$P(C= 1 | \ell) = \rho_\ell$.
%\end{equation}
\normalsize
Given the click event is always observed, $\rho_\ell$ can be estimated using MLE. The likelihood for the parameter $\rho_l$ can be written as:
\small
\begin{equation}
\mathcal{L}(\rho_\ell) = \prod_{c_i \in S_c} \rho_l^{c_i} (1 - \rho_\ell)^{1-{c_i}}
\label{eqn:likelihood}
\end{equation}
\normalsize
where $S_c$ is the set of clicks and $c_i$ is the value of the click of the $i^{th}$ occurrence for position $\ell$. By taking the log of (\ref{eqn:likelihood}),
calculating its derivative and equating it to zero, we get the MLE estimation of the parameter $\rho_\ell$. In this case, it is the sample mean of $c_i$'s:
\small
\begin{equation}
\rho_\ell = \frac{\sum_{c_i \in S_c} c_i}{\mid S_c \mid}
\end{equation}
\normalsize
\subsection{Probit Regression Model}. The CTR-based method is very intuitive but does not consider actions' features and their probability of being clicked. Furthermore, it can incur in the same
bias-related problem of the naive rankers since the clicks will likely be more frequent towards the beginning of the ranking.
We aim to learn a mapping $x \rightarrow \left[0,1\right]$ from a set of
features $x$ to the probability of a click. Bayesian Linear Probit model is a generalized linear model (GLM) with a Probit link function. The sampling distribution is given by:
$P(C | \theta, x) := \Phi (C \cdot \theta^T x / \beta)$, where we assumed that $C$ is either 1 (click) or 0 (no click) and $\Phi$ is the cumulative density
function of the standard normal distribution: $\Phi(t) := \int_{-\infty}^{t} \mathcal{N}(s; 0,1) ds$. It serves as the link function that maps the output of the
linear model (sometimes referred to as the score) in $\left[-\infty, \infty\right]$ to a probability distribution in $\left[0,1\right]$ over the observed data, $C$.
The parameter $\beta$ scales the steepness of the inverse link function. \\
The function $P(C | \theta, x)$ is called likelihood as a function of $\theta$ and sampling distribution as a function of $C$; the latter is the generative model
of the data and is a proper probability distribution whereas the former is the weighting that the data, $C$, gives to each parameter. The model uncertainty
over the weight vector is captured in $P(\theta) = \mathcal{N}(\mu,\sigma^2)$. Given a feature vector $x$, the proposed sampling distribution together with
the belief distribution results in a Gaussian distribution over the latent score. Given the sampling distribution $P(C | \theta, x)$  and the prior $P(\theta)$,
the posterior is calculated: $P(\theta | C, x) := P(C | \theta, x) P(\theta)$. \\
We keep a Probit Regression (PR) model for each position $\ell$. Given the likelihood $P(C | x, a, \ell, \theta)$, the posterior is calculated as: \\
$P(\theta | C, x, a, \ell) := P(C | x, a, \ell, \theta) P(\theta)$. \\
Then, the predictive distribution $P(C | x, a, \ell)$ can be computed with given feature vector and
the posterior (See~\cite{graepel2010web} for details). As we mentioned in Section~\ref{sec:PBM}, the probability of getting a click on action $a$ in position
$\ell$ is equal to $P(C=1 | x, a, \ell) = P(E=1 | \ell) P(R=1 | x, a)$. Here, our goal is to compute $q_\ell = P(E=1 | \ell)$ and we compute it as:
\small
\begin{align*}
q_\ell = \frac{P(E=1 | \ell) P(R=1 | x, a)}{P(E=1 | \ell=1) P(R=1 | x, a)}
\end{align*}
\normalsize
where we assume $P(E=1 | \ell=1) = 1$.
It has to be noted that although this assumption holds for the applications considered in the experimental section of this paper, this is not guaranteed in all
real-world applications. It is possible that the content on the page may get reshuffled by another system and the position of the component visualizing the
ranking changed, with significant impact on the chance for the customer to observe the content.
In Section~\ref{sec:experiments}, we will provide experimental results where this assumption is violated.

\subsection{Expectation-Maximization}.
After the observations made for the CTR and PR estimators,
we to explore different directions in order to provide a solution
that can be more robust in
real-world scenarios.
The Expectation-Maximization (EM) algorithm can be applied to a large family of
estimation problems with latent variables. In particular, suppose we have
a training set $X=\{x_1, \dots , x_n\}$ consisting of $n$ independent examples.
We wish to fit the parameters of a model $P(X, Z)$ to the data, where the
likelihood is given by:
\small
\begin{align}
\mathcal{L}(\theta) & =  \sum_{i=1}^n \log P(X; \theta) =  \log \sum_Z P(X, Z; \theta)
\end{align}
\normalsize
However, explicitly finding the maximum likelihood estimates of the parameters $\theta$ may be hard. Here, the $z^{(i)}$'s are the unobserved latent random
variables. In such a setting, the EM algorithm gives an efficient method for maximum likelihood estimation. To maximize $\mathcal{L}(\theta)$, EM construct a
lower-bound on $\mathcal{L}$ (E-step), and then optimize that lower-bound (M-step) repeatedly.
The EM estimator provided in this section can be seen as a generalization of PR estimator, which should provide better practical performance. Given the relevance estimate $\gamma_{x,a}=P(R=1|x,a)$,
the position bias $q_l=P(E=1|l)$ where $P(C=1|x,a,l)=P(E=1|l)P(R=1|x,a)$,
and a regular click log $\mathcal{L} = \{(c, x, a, \ell)\}$, the log likelihood of generating this data is:
\begin{equation}
\log P(\mathcal{L}) = \sum_{(c, x, a, \ell) \in \mathcal{L}} c \log q_\ell \gamma_{x,a} + (1 - c) \log (1 - q_\ell \gamma_{x,a})
\end{equation}
The EM algorithm can find the parameters that maximize the log-likelihood of the whole data. In \cite{wang2018position}, the authors
introduced an EM-based method to estimate the position bias from regular production clicks. The standard EM algorithm iterates over the Expectation and
Maximization steps to update the position bias $q_\ell$ and the relevance parameter $\gamma_{x,a}$.
In this paper, we modify the standard EM and take $\gamma_{x,a}$ equal to $\sigma(A_t^T \hat{\theta}_t)$ at each step $t$ where $A_t$ is the contextualized
action. In this way, we take the context information into account. At iteration $t+1$, the Expectation step estimates the distribution of hidden variable $E$ and $R$
given parameters from iteration $t$ and the observed data in $\mathcal{L}$:
\begin{equation}
P(E=1 | c, x, a, \ell) = \frac{ (1 - \gamma_{x,a}^{(t)}) q_\ell^{(t)}}{1 - q_\ell^{(t)} \gamma_{x,a}^{(t)}}
\end{equation}
Fore more details we refer to \cite{wang2018position}.
The Maximization step updates the parameters using the quantities from the Expectation step:
\small
\begin{equation}
q_\ell^{(t+1)} = \frac{1}{T} \sum_{t} \left( c_\ell^{(t)} + (1 - c_\ell^{(t)}) \frac{ (1 - \gamma_{x,a}^{(t)}) q_\ell^{(t)}}{1 - q_\ell^{(t)} \gamma_{x,a}^{(t)}} \right)
\end{equation}
\normalsize
%

% \section{Experiments}
\label{sec:experiments}
% !TEX root = paper.tex

In this section we provide a number of empirical results to demonstrate the advantages of the proposed algorithms compared to their ``naive'' counterparts and other baselines.
To this aim, we perform experiments on synthetic datasets with different variants of the position bias estimation in a controlled environemt showing differences that would not be possible to measure in an online environment.
We also tested our algorithms in online experiments against other baselines but in order to avoid the risk of a negative impact on the customer experience, we ran online experiments
comparing our two variants of the presented algorithms and two ``safe'' production-like baseliense.
Comparing our algorithms to baselines that provided negative results in the offline experiments
would be irresponsible and completely against the customers' interest.

\begin{table*}[h!]
\centering
 \begin{tabular}{r | c | c | c | c  }
                                                                 & \multicolumn{4}{| c }{Number of actions/positions}                       \\
\hline
  ALGORITHMS                                        &    1                                       &                   5                           &       10                                    &              20                              \\
\hline
  LinUCB                                                   &   48278.30 $\pm$ 5.99    &   69334.81 $\pm$ 144.02    &   68294.01 $\pm$ 171.0     &   65185.37  $\pm$ 415.50        \\
  LinUCB-PBMRank (Real)                      &   48274.81  $\pm$ 3.51     &   75800.33 $\pm$ 5.58       &   76296.08 $\pm$ 6.53      &   76307.25 $\pm$ 7.86             \\
  LinUCB-PBMRank (EM)                        &   48266.04 $\pm$ 10.63   &   74119.72 $\pm$  219.15    &   74364.61 $\pm$ 133.72   &   74567.43 $\pm$ 273.93        \\
  LinUCB-PBMRank (PR)                         &   48276.28 $\pm$ 3.51     &   75786.47 $\pm$ 22.45     &   76291.33  $\pm$  6.01      &   76102.86 $\pm$ 40.49         \\
  LinUCB-PBMRank (CTR)                       &   48247.31 $\pm$ 5.90    &    72104.12 $\pm$ 78.13      &    72981.72  $\pm$ 132.31   &   73610.40 $\pm$ 350.70       \\
\hline
  LinTS                                                       &   48518.36 $\pm$ 12.93    &  67234.41 $\pm$ 209.95  &  69363.75 $\pm$ 327.70   &   68054.39  $\pm$ 124.57      \\
  LinTS-PBMRank (Real)                          &   48559.73 $\pm$ 9.27     &  76194.49  $\pm$ 1.95       &  76709.01 $\pm$ 2.45        &   76803.95  $\pm$ 11.74         \\
  LinTS-PBMRank (EM)                            &   48153.47  $\pm$ 49.21   &  74992.00 $\pm$  91.78    &  75365.20  $\pm$ 167.04   &   76018.56  $\pm$  96.39        \\
  LinTS-PBMRank (PR)                             &   48559.73 $\pm$ 9.27     &  73690.61  $\pm$ 49.48    &  74670.29  $\pm$  55.35   &   75940.56  $\pm$  245.67       \\
  LinTS-PBMRank (CTR)                          &   48523.96 $\pm$ 6.31      &  74254.12  $\pm$ 23.55    &  74682.50  $\pm$ 186.22  &   73128.87   $\pm$ 312.81       \\
\hline
  Random Selection                                 &   45044.14 $\pm$ 11.60  &   70780.10 $\pm$ 20.29    &  71253.58  $\pm$ 21.74   &   71273.89  $\pm$  92.59           \\
 \end{tabular}
 \caption{Cumulative reward on \textbf{SINREAL}. \label{t:cum_real_reward}}
\end{table*}
\begin{table*}[h]
\centering
 \begin{tabular}{r | c | c | c | c  }
                                                                 & \multicolumn{4}{|  c }{Number of actions/positions}                       \\
\hline
  ALGORITHMS                                        &    1                   &          5           &   10                 &   20                      \\
\hline
  LinUCB                                                  &   34790.07  $\pm$ 5.55      &   49453.96 $\pm$ 430.82  &   50450.33  $\pm$ 65.66   &   20915.29  $\pm$  2165.77     \\
  LinUCB-PBMRank (Real)                     &   34682.12  $\pm$  31.15     &   53081.15  $\pm$ 148.92  &   53436.85  $\pm$  19.28   &   53538.65 $\pm$  344.03      \\
  LinUCB-PBMRank (EM)                       &   34568.03  $\pm$ 173.82   &   50562.10 $\pm$ 113.04   &   50069.72  $\pm$ 203.47  &   51397.31  $\pm$  221.55      \\
  LinUCB-PBMRank (PR)                        &   34584.33 $\pm$ 24.68     &   53178.26 $\pm$ 139.11    &   53352.61  $\pm$ 148.24   &   53470.40 $\pm$ 168.91       \\
  LinUCB-PBMRank (CTR)                     &   34552.33  $\pm$ 124.94   &   50029.49  $\pm$ 132.16   &  50570.73  $\pm$ 225.96   &  49521.61 $\pm$ 302.92       \\
\hline
  LinTS                                                     &   34850.33 $\pm$ 130.12   &  46402.29 $\pm$ 83.20    &   47707.79 $\pm$ 90.08     &  39598.70 $\pm$ 1228.26        \\
  LinTS-PBMRank (Real)                         &   34882.33 $\pm$ 117.58   &  53700.02  $\pm$ 62.98   &   53945.51 $\pm$ 73.62     &  53939.03 $\pm$  133.72        \\
  LinTS-PBMRank (EM)                           &   34795.66 $\pm$ 113.52  &  47921.58  $\pm$ 26.23    &   48176.78 $\pm$  235.97   &  51994.30 $\pm$  103.47       \\
  LinTS-PBMRank (PR)                            &   34882.33 $\pm$ 35.12   &  51877.87  $\pm$  53.50    &   52283.41 $\pm$ 126.40   &  52444.39 $\pm$  122.09       \\
  LinTS-PBMRank (CTR)                          &   34766.01 $\pm$ 171.48  &  47657.17  $\pm$ 58.34    &   46218.11 $\pm$ 189.21     &  45815.60 $\pm$  256.34       \\
\hline
  Random Selection                                 &   26721.66 $\pm$ 35.42  &  42139.62 $\pm$ 89.18   &   42389.31 $\pm$ 126.45  &    42552.10 $\pm$  85.47       \\
 \end{tabular}
 \caption{Cumulative reward on \textbf{SINBIN} datasets. \label{t:cum_bin_reward}}
\end{table*}
\section{Experiments}
\begin{table*}[t]
\centering
 \begin{tabular}{r |  r  | c | c | c | c}
\hline
                                                           &  ALGORITHMS
                                                           & 5 / SINBIN
                                                           & 5 / SINREAL
                                                           & 10 / SINBIN
                                                           & 10 / SINREAL \\
\hline
\multirow{3}{*}{$\epsilon$=0.1}     &  LinTS-PBMRank (Real)  &  48348.93 $\pm$ 61.61   & 68579.71 $\pm$ 14.53  & 48549.76 $\pm$ 87.88  & 69033.46 $\pm$ 43.39 \\
                                    &  LinTS-PBMRank (EM)    &  46175.04 $\pm$ 150.30  & 67962.91 $\pm$ 56.57  & 46574.36
$\pm$ 178.47 & 68027.64 $\pm$ 158.46 \\
                                    &  LinTS-PBMRank (PR)    &  47885.84 $\pm$ 234.51  & 66564.25 $\pm$ 50.06  & 47987.43
$\pm$ 267.57 & 67149.27 $\pm$ 302.14 \\
\hline
\multirow{3}{*}{$\epsilon$=0.25}   & LinTS-PBMRank (Real)    &  39538.48 $\pm$ 103.30  & 57147.24 $\pm$ 32.55  & 39787.77 $\pm$ 183.72  & 57528.76 $\pm$ 111.05  \\
                                   & LinTS-PBMRank (EM)      &  38164.60 $\pm$ 380.26  & 55250.56 $\pm$ 232.94 & 38649.44
$\pm$ 384.07  & 55949.08 $\pm$ 234.69 \\
                                   & LinTS-PBMRank (PR)      &  37821.41 $\pm$ 504.89  & 52039.96 $\pm$ 398.44 & 37838.06
$\pm$ 534.24  & 53135.97 $\pm$ 397.89 \\
\hline
\multirow{3}{*}{$\epsilon$=0.5}    & LinTS-PBMRank (Real)    &  25546.38 $\pm$ 109.99  & 38095.41 $\pm$ 43.26 & 25811.65
$\pm$ 197.66 & 38332.70 $\pm$ 50.64  \\
                                   &  LinTS-PBMRank (EM)     &  25051.16 $\pm$ 128.68  & 37153.18 $\pm$ 382.06 & 25336.41
$\pm$ 540.04 & 37637.62 $\pm$ 127.40  \\
                                   &  LinTS-PBMRank (PR)     &  23901.48 $\pm$ 408.30  & 34382.20 $\pm$ 313.71 & 23958.86
$\pm$ 708.51 & 34689.39 $\pm$ 432.71 \\

 \end{tabular}
 \caption{Cumulative reward on \textbf{SINBIN} when position bias is $\frac{1-\epsilon}{\exp(position)}$. \label{t:cum_pb_eps}}
\end{table*}

\begin{figure}[h!]
\hspace{-2mm}
\begin{subfigure}[b]{0.5\textwidth}
\centering
\includegraphics[scale = 0.50] {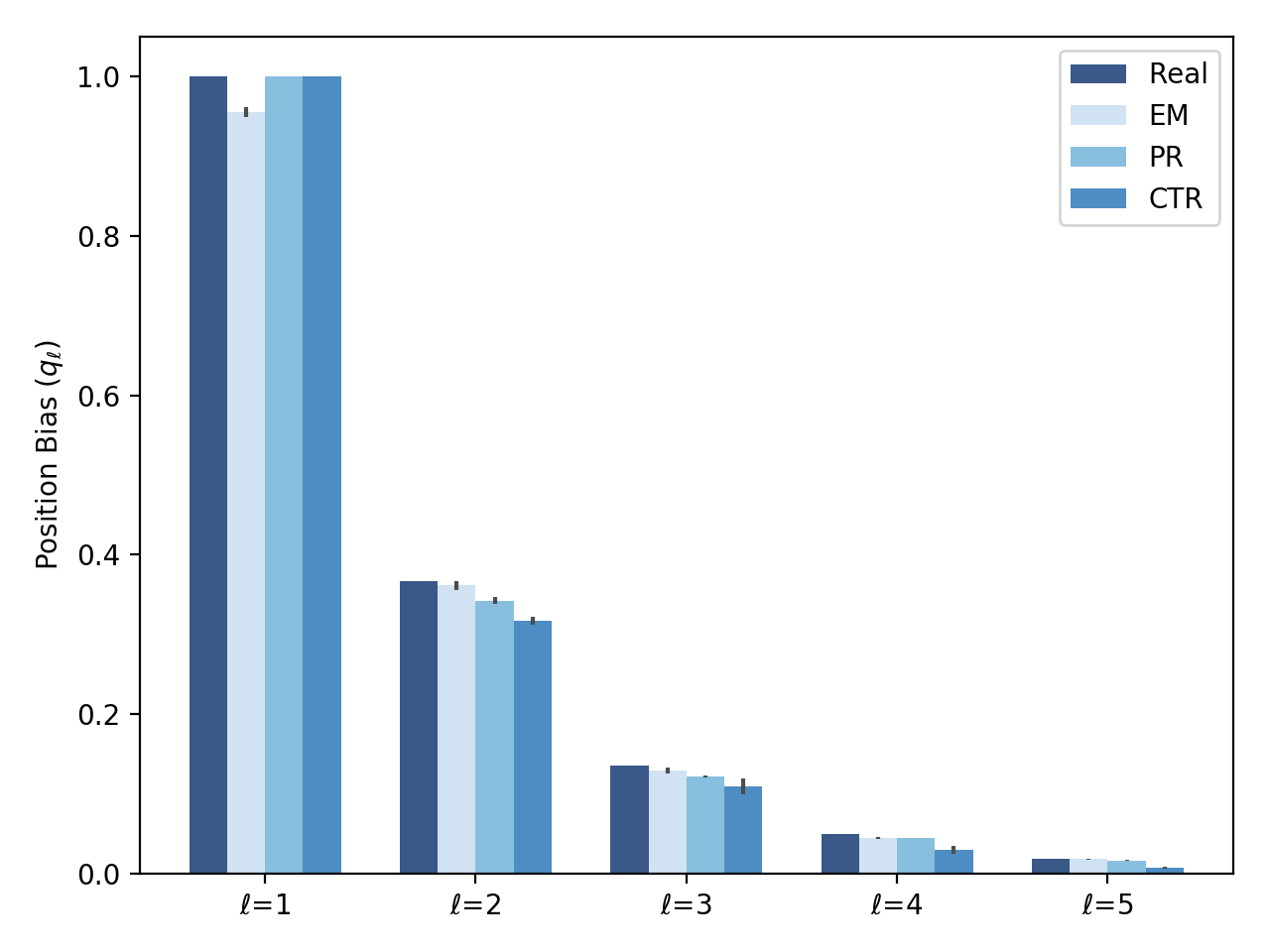}
\label{subfig:bin_pb_est5}
\end{subfigure}
\hspace{3mm}
\begin{subfigure}[b]{0.5\textwidth}
\centering
\includegraphics[scale = 0.50] {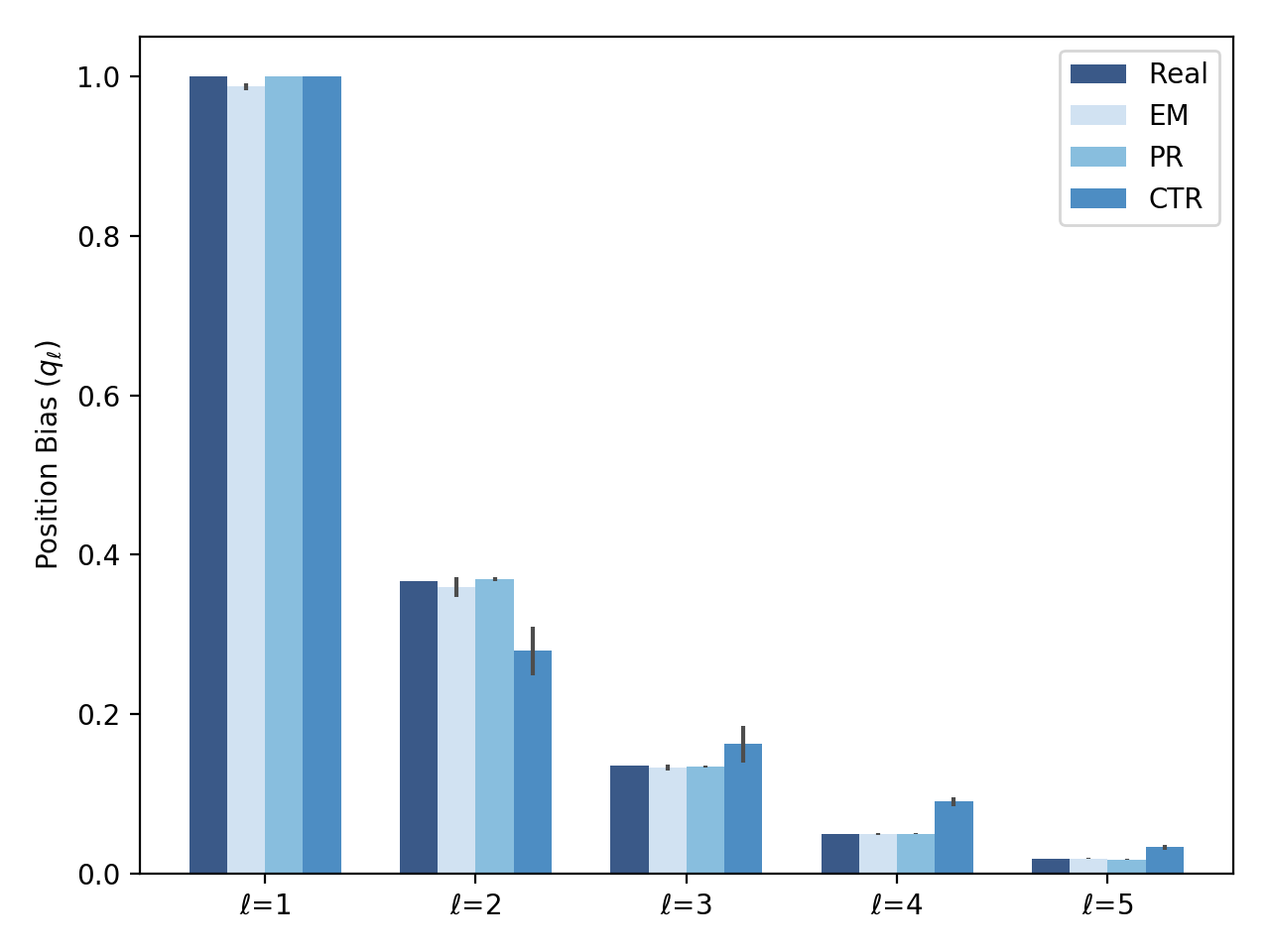}
\label{subfig:real_pb_est5}
\end{subfigure}
\caption{\footnotesize{Comparison of the real position biases and the position biases estimated by CTR, PR and EM methods for the top 5 positions. SINBIN on the left and SINREAL on the right.}}
\label{fig:estimationComp5}
\vspace*{-2mm}
\end{figure}
\subsection{Offline Experiments}
\label{sec:offline_experiments}
For the purpose of testing our algorithms in a controlled environment, we created two
synthetic datasets with 25 available actions and we limited the algorithm to select
a maximum of 20 actions, simulating the behavior of a page that does not display
all the actions to all the customers.
The actions vectors were generated as in the following: we fixed the number of dimensions to 5 and then generated dense vectors of random numbers
in [0,1) and then set all the entries having a value below 0.1 to 0.0 (introduces some sparsity).
The context vectors, part of the same datasets, are set to have 10 dimensions and generated in the same way of the actions.\\
A simplified version of the behaviour of the production system is reproduced in the offline experiments,
so we join the action and context vectors to create a \textit{contextualized action}.
This is created by concatenating the action vector, the context vector and the
vectorized outer product of the two. This process generates 25 vectors, each
one representing an action and each vector is made of three blocks: the action
vector, the context vector, and the cross product between the action vector and
the context vector.
After the vectors are generated, they get normalized by dividing them by their respective square norms.
The only vectors received by the predictors are the ones made available at the end of this process.\\
For the dataset with \textbf{real valued rewards} (later called \textbf{SINREAL}), the rewards are generated as follows: at the beginning of the process a
unit length random vector $w$ is fixed and $w$ will be used to compute the inner product with the contextualized actions, following the linear assumption
made in  Section~\ref{sec:learning_model}.
The reward is generated by summing the inner product between $w$ and the contextualized action vector with a noise factor uniformly sampled in the interval
[-0.1,0.1). Then, we apply floor and ceiling operations to make sure to obtain a reward in [0,1]. In the case of the dataset with \textbf{binary valued rewards} (later
called \textbf{SINBIN}), the same procedure is followed but we binarize the rewards by  thresholding with a predefined hyperparameter.
%(if the value of the reward is greater than $\tau$ it is 1, otherwise 0).\\
%
Before providing the rewards to update the predictor, the rewards are divided by the exponential of the position assigned to the corresponding
action by the learning algorithm (this is done ``online'' and depends on the predictions made by the algorithm).
The aim is to mimic the behavior observed in online experiments where the users tend to click significantly more on the
top positions on the ranking. The exponential function was chosen after observing the behavior of customers in some online experiments.\\
\subsubsection{Results}
In our experiments, we compared the two algorithms presented in Section~\ref{sec:linucbrank}  with their counterparts that do
not account for the bias introduced by the ranking position namely LinTS and LinUCB. These algorithms select the actions taking the top-K with the highest
scores instead of the single best one as in their original definition. The update operation is performed using all the selected actions and the corresponding
rewards without any re-weighting. This is equivalent to set all the $\{q_\ell\}_{\ell=1}^L$ to 1 in the algorithms referenced above.\\
\textbf{Synthetic Data Results}.
Tables~\ref{t:cum_real_reward} and \ref{t:cum_bin_reward} report the results
of experiments run on
synthetic data in order to validate our ideas in a controlled environment.
The dataset used in this section are SINREAL and SINBIN, whose details are available in Section~\ref{sec:offline_experiments}.
Please note that the since the datasets are generated artificially,
every potential prediction of the algorithm can receive the correct reward and
we do not need to employ techniques for running offline evaluation with
biased datasets (e.g., \cite{li2018offline}).
In these offline experiments, we can observe two important trends:
i) not addressing the position bias can significantly mislead algorithms to the point that they can become worse than a random selection,
ii) using an automatic method for estimating the position bias gives a clear advantage but there is no clear winner between PR and EM.\\
\textbf{Position Bias Estimation Results}.
The previous experiments show that CTR is inferior as position bias
estimation method, while PR and EM perform almost equally.
In Figure~\ref{fig:estimationComp5} we compare the quality of the
estimation methods by comparing the esimated position biases with the true
values observed in the synthetic datasets.
However, it is important to recall that for the CTR and PR estimators
the parameter for the first position is artificially set to 1,
while the EM method is performing its estimation without any additional information.
This is particularly useful in cases where the hyperparameter associated
with the first position is unknown because it is controlled by external factors (e.g., the ranked content
is displayed in a position where deos not catch the attention of the users).
We conducted a range of experiments, reported in Table~\ref{t:cum_pb_eps},
to assess the sensitivity of PR with respect to this parameter.
The results clearly show that the more severe the violation of the $q_1 =1$ becomes, the better EM becomes compared to PR.
% First, we provide the results of experiments conducted to test the quality of the position bias hyperparameters estimated by different algorithms. PE dont get this sentence
\subsection{Online A/B Experimentation}
\begin{figure}
    \centering
    \begin{minipage}{0.5\textwidth}
        \centering
        \includegraphics[width=0.95\textwidth]{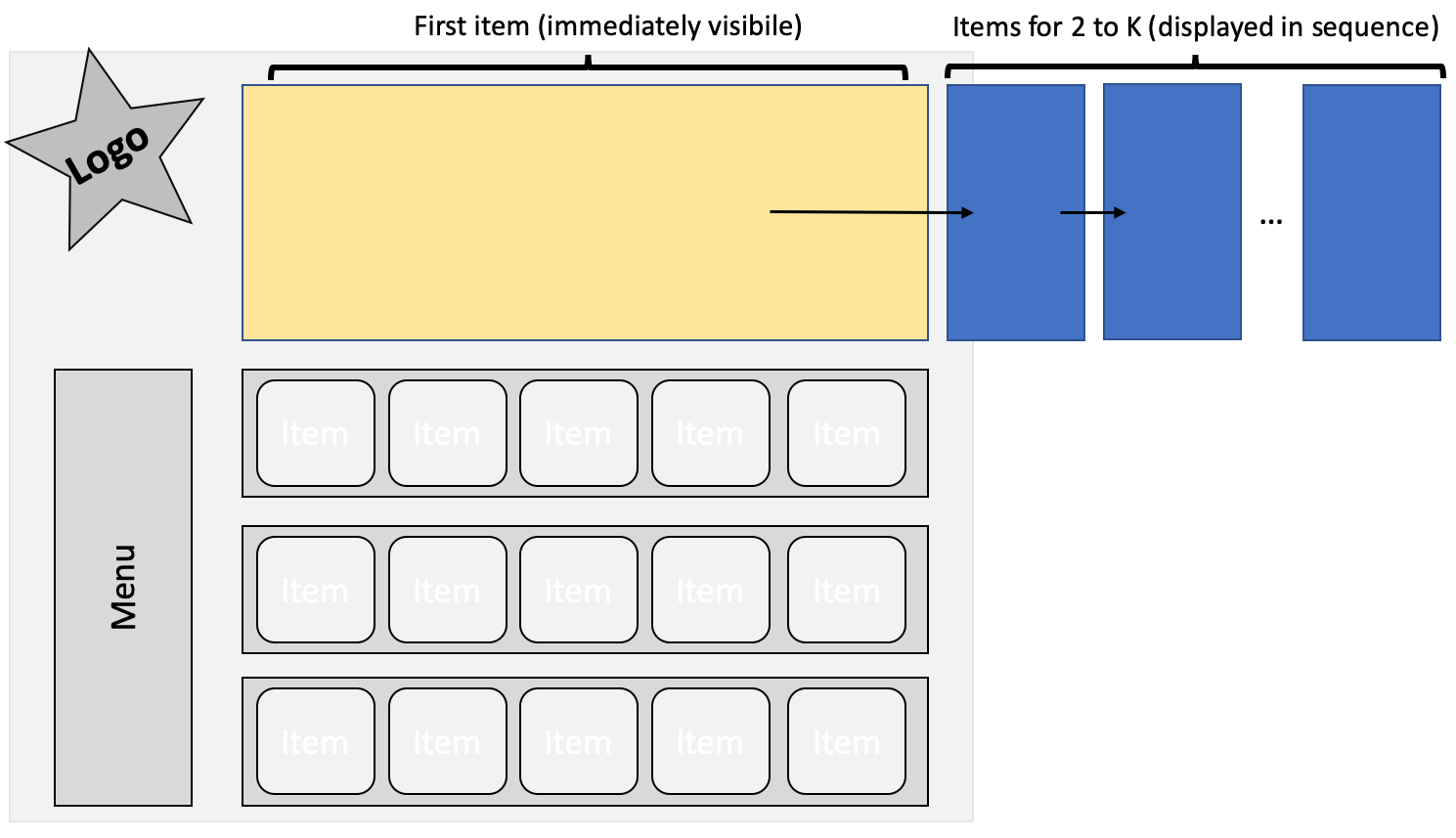} % first figure itself
        \caption{Structure of the page where the experiment was ran. The central-top slot is optimized using the methods described in this paper.
        All the items in the list get eventually displayed on the page since the slot is automatically loading the next piece of content
        after a fixed amount of time. }
        \label{audiobooktrends}
    \end{minipage}
\end{figure}

To validate our offline results and to show the effectiveness of our approach in a real-world
scenario, we conducted two end-customer facing online A/B tests. Due to the costs
and potential negative customer experience of running A/B tests involving real
paying customers, we focused on two main scenarious.
In each scenario we pick one widget, which is a so-called carousel\footnote{
  A carousel consists of a list of
banners, where only one banner is displayed at a time and rotated to the next one
after a certain time period.
}
UI and is embedded at the top of the landing page of a large music streaming service.

We alter the arrangement of the list items
between control A and treatment B to test different baselines against
configurations of our bandit-based ranking approach.
Particularly we test:
\begin{itemize}
  \item a human-curated list arrangement in control against our approach
  with fixed position biases, i.e. without online automatic estimation
  \item a collaborative filtering based ranking in control against our approach
  with online EM position bias estimation as treatment
\end{itemize}
The customers are split equally, 50\%/50\% random allocation, between control and
treatment.

% !TEX root = paper.tex
\subsubsection{Online Learning to Rank A/B Test vs. Human-curated content}
\label{sec:experiment:humancurated}
In this experiment the goal is to have a confined test for the bandit learning
to rank algorithm and thus we purposefully do not include automatic position
bias estimation. Instead, we rely on manual hyper parameters based on
view events, where the parameter for position $i$ is based on the number of
historical customer requests that viewed $i$ divided by the number of requests.
The candidate widget consists of $50$ candidate items,
which were represented by banners containing music spanning different genres and user tastes (e.g., audio books, music for children).
Our control treatment always shows the same order of $13$ manually curated items
to customers. In treatment, we apply our ranking bandit to contextually rerank
the candidate set every time the customer visits the landing page. We pick the
top-13 scored banners to fill the carousel and present them to the customer.
To contextualize the ranking, we leverage different types of features representing the customer, content, and general context,
such as temporal information, customer taste profiles and customer affinities towards musical content.
We see major increases of various classical ranking measurement and engagement metrics
in treatment which leverages the ranking bandit.
Overall, here customers interacted more with the widget and also consumed more music.
In particular, if we compare the performance of the widget
with the version provided to the control group, 
we are able to improve the following widget specific metrics:
\begin{itemize}
  \item the mean reciprocal rank (MRR) increased by $15.38\%$,
  \item the amount of attributed playbacks increased by $17.16\%$,
  \item the listening duration measured in log seconds increased by $16.90\%$,
  \item the number of customers playing music increased by $15.62\%$.
\end{itemize}
All results are statistical significant with p-values below $0.001$.
These positive results are also found in the performance of the landing page
which contains the widget:
\begin{itemize}
  \item there is $2.33\%$ more playbacks originating from the landing page,
  \item the listening duration measured in log seconds increased over all customers increased by $3.04\%$,
  \item the number of customers who played music increased by $2.81\%$.
\end{itemize}
All results are statistical significant with p-values below $0.001$.
We also tracked the position of each banner in the carousel during the course of
this experiment, which revealed that the approach is able to respond to
intra-day short-term trends. For instance, specific types of contents are
popular
only at some time of the day and the algorithm is able to learn that. Such a case is
depicted in Figure~\ref{nichegenre}, which plots the average ranks over time.
As we can see, there is specific content which is popular during night time and 
the ranking bandit is able to capture intra-day trends thanks
to temporal features provided as part of the context and fast model updates.
Additionally, we observed that the ranking bandit was able to handle seasonal
content: an example is shown in Figure~\ref{fathersday}, which shows the average position
of a banner targeted to Father's day (celebrated on May 30th) in Germany that was
ranked high in the days leading to the holiday.

\begin{figure}
    \centering
    \begin{minipage}{0.5\textwidth}
        \centering
        \includegraphics[width=0.95\textwidth]{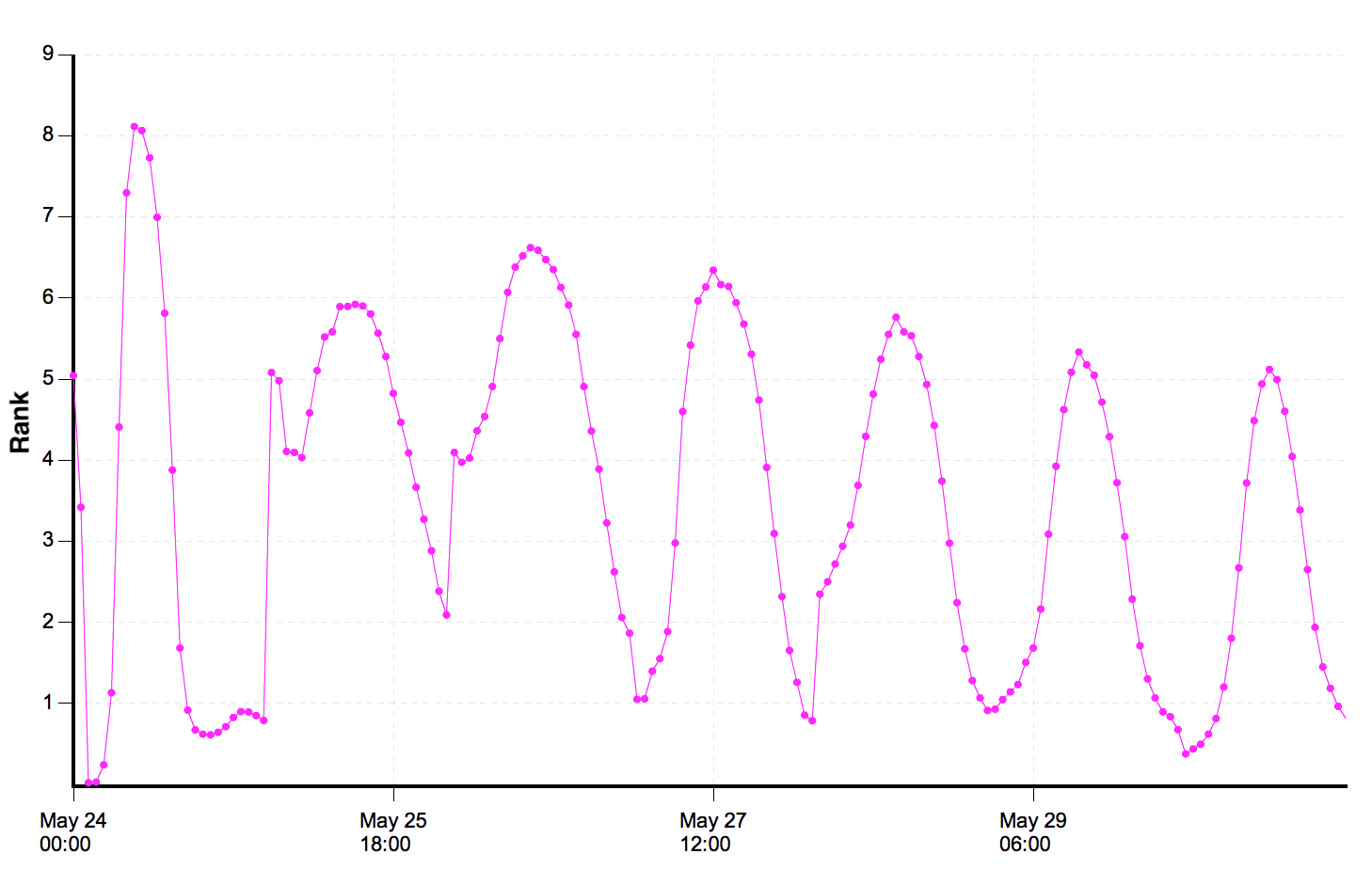} % first figure itself
        \caption{Intra-day trends for audio content of niche genre.}
        \label{nichegenre}
    \end{minipage}\hfill
    \begin{minipage}{0.5\textwidth}
        \centering
        \includegraphics[width=0.95\textwidth]{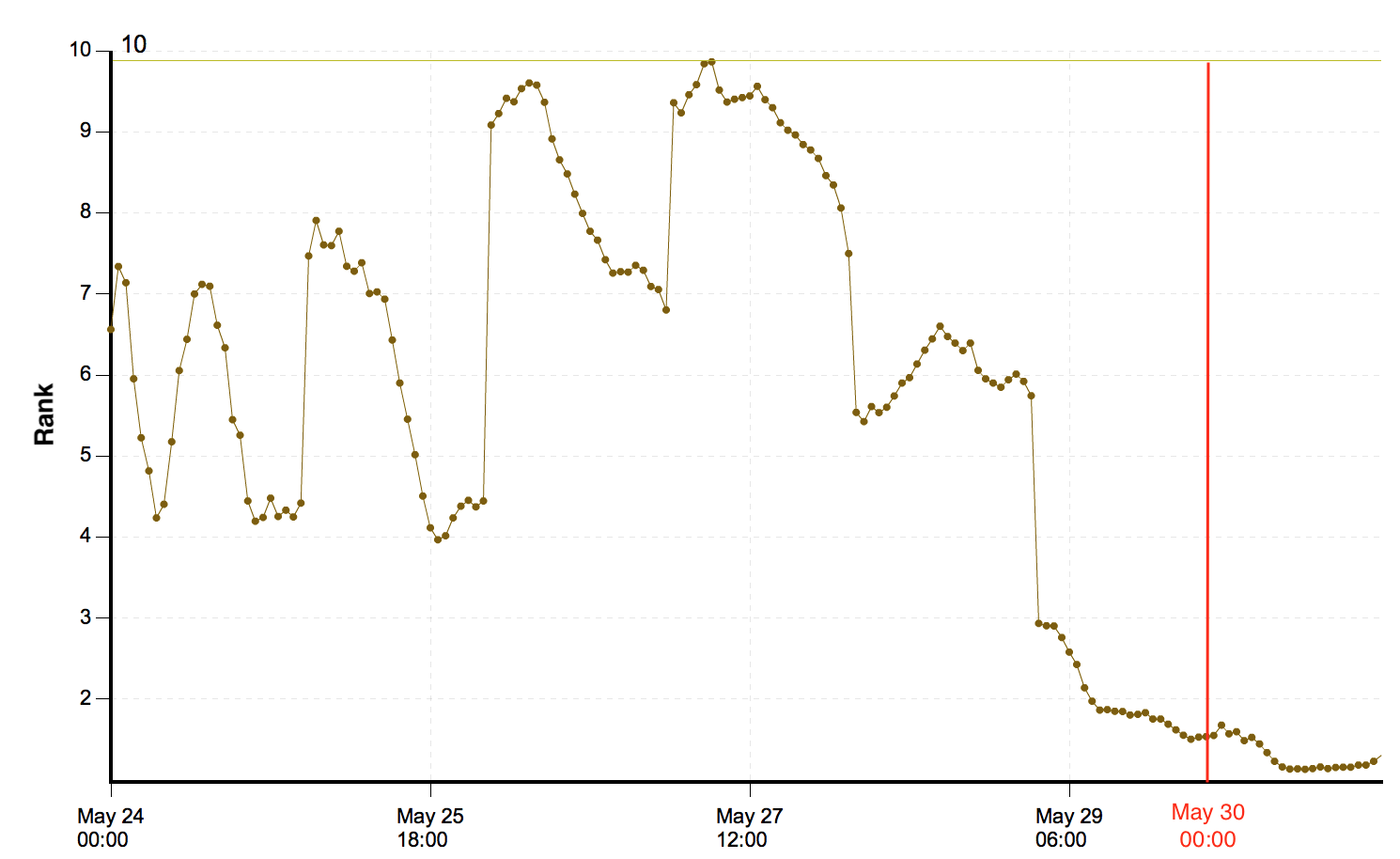} % second figure itself
        \caption{Intra-day and intra-week trend for an item regarding music for
        Father's day. There is an evident trend in the days leading to the
        holiday where the item becomes more popular.}
        \label{fathersday}
    \end{minipage}
\end{figure}

\subsubsection{Online Learning to Rank A/B Test vs. Matrix factorization baseline}
In this experiment, we tested the ranking bandit with position
bias estimation against a
matrix factorization baseline on the carousel widget during summer 2019 over
8 days in the US\@. The carousel contained
10--15 banners that were manually curated and changed over the time of the
experiment. In the control group, the banners were ordered by scores derived
from an existing production system that is based on matrix factorization. In the
treatment, we applied the ranking bandit with position bias estimation. To contextualize,
we used temporal features, as well as several features to represent the customer
such as the customer's taste profile and the scores from the matrix
factorization baseline.

Overall, we saw an increased customer engagement in treatment compared to
control. In particular, we saw improvements in
the treatment along the following metrics for the targeted widget:

\begin{itemize}
  \item the mean reciprocal rank (MRR) increased by $5.08\%$,
  \item the attributed playbacks increased by $7.57\%$,
  \item the listening duration measured in log seconds increased by $7.23\%$,
  \item the number of customers playing music from this widget increased by $6.72\%$,
\end{itemize}
All results are statistical significant with p-values below $0.001$.
We also improved metric for the whole landing page that contains the widget:
\begin{itemize}
  \item the number attributed playbacks increased by $0.8\%$ with p-value=$0.075$
  \item the listening duration measured in log seconds increased by $0.96\%$.
  \item the number of customers who played music increased by $0.92\%$.
\end{itemize}

All results for which statistical significance is not specified
were significant with p-values below $0.001$.
Finally, we observed similar to Section~\ref{sec:experiment:humancurated} that
the ranking bandit was able to capture intra-day trends, where it ranked a summer
playlist higher during the day and evening than at night, and sudden customer trends,
where it learned within 1 day to rank high a banner featuring a new track by
a famous American artist. In both cases, the matrix factorization baseline missed these
trends. See figures~\ref{fig:summer_playlist} and~\ref{fig:singerUS} for more
details.

\begin{figure}
    \centering
    \begin{minipage}{0.5\textwidth}
        \centering
        \includegraphics[width=0.75\textwidth]{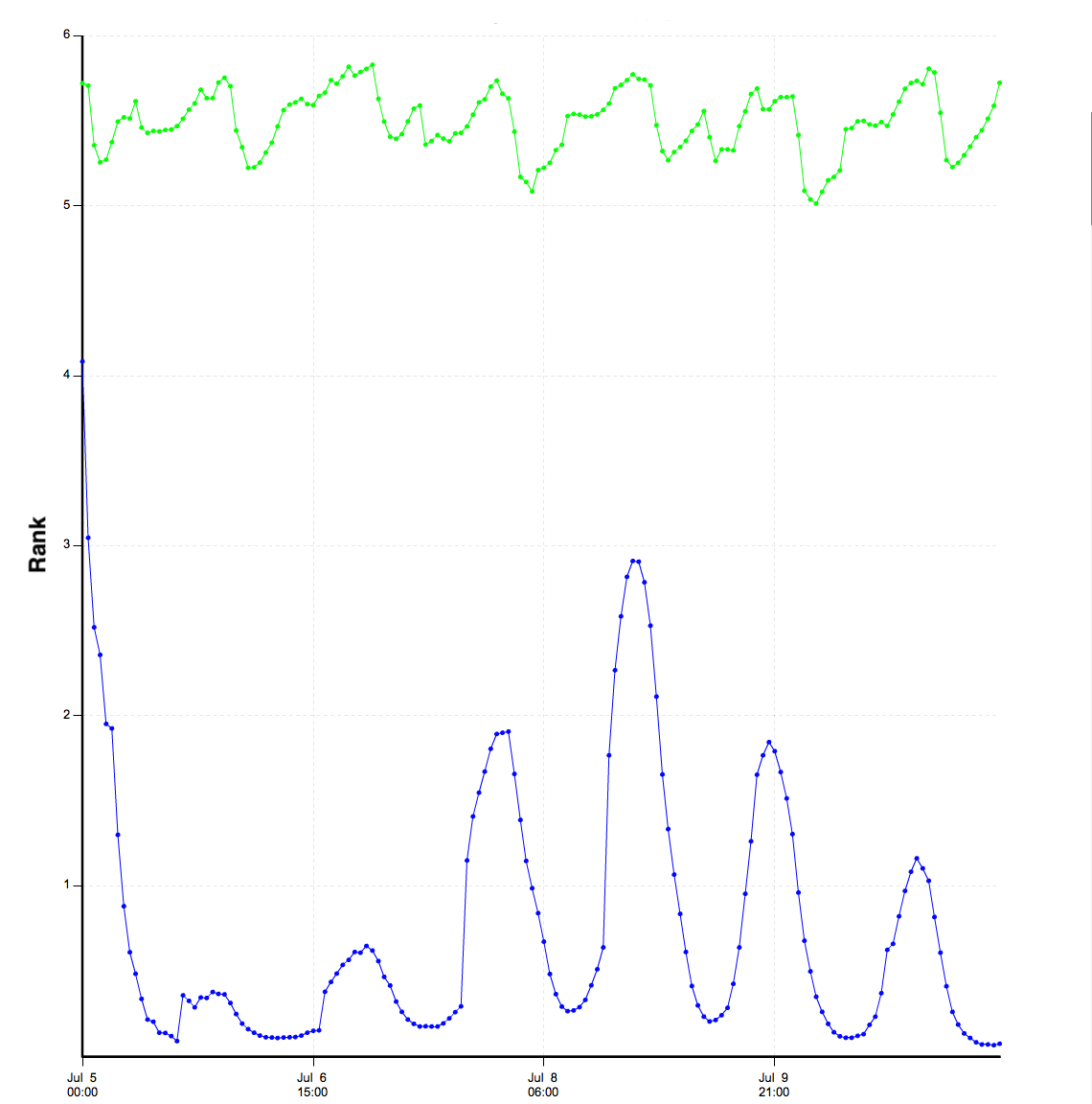}
        \caption{Intra-day trends for summer playlist in the beginning of July.
        The ranking bandit is in blue and the baseline recommender in green.
        The ranking bandit is able to catch the general trend earlier than the
        baseline recommender and also to follow the intra-day fluctuations.}
        \label{fig:summer_playlist}
    \end{minipage}\hfill
    \begin{minipage}{0.5\textwidth}
        \centering
        \includegraphics[width=0.75\textwidth]{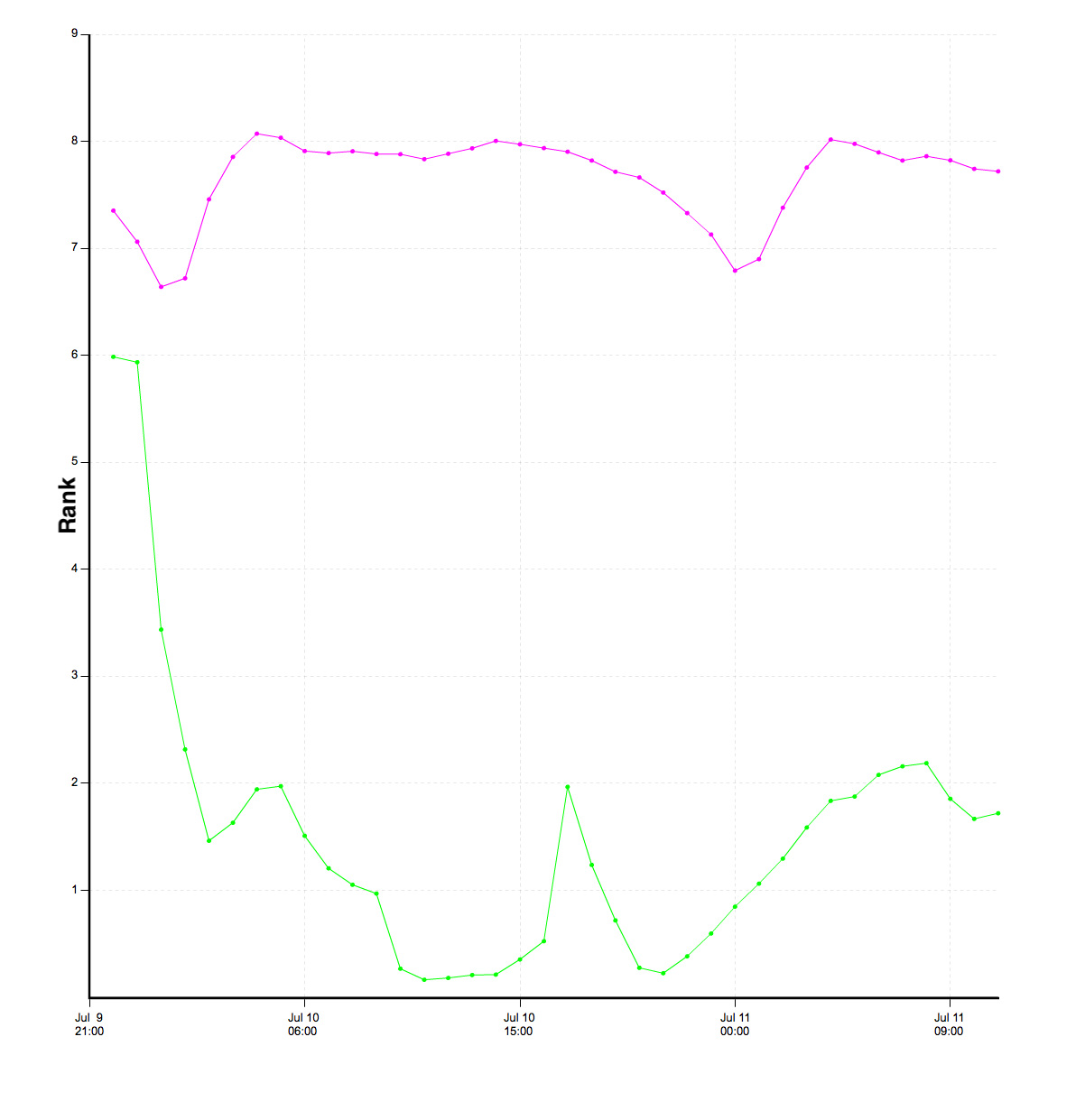}
        \caption{Trend for for banner featuring a new track by a well-known American singer.
        The ranking bandit is in green and the baseline recommender in pink.
        As it often happens recently released content by popular artists catches
        the attentions of customers outside the core artist fan base. In this
        plot it is evident that the ranking bandit catches the trend much
        earlier than the baseline recommender.}
        \label{fig:singerUS}
    \end{minipage}
\end{figure}

\section{Lessons learned}
\label{sec:lessons}
While most of the risks where mitigated before the deployment
and everything move quite smoothly, there are a few facts which
we considered surprising.

\subsection{Usage of one-hot encoding} 
We developed methods which leverage ``contextualized actions'' allowing us to
perform an extensive amount of features engineering. In this way
we can leverage highly non-linear model trained on historical information
to produce high-quality features. In the online experiments we reported
in this paper, we used the our system to re-rank a very small pool
of items (represented by a large image) linked to a piece of musical
content. Turns out that the one-hot-encoding representation of the items
combined with the context by the mean of the cross product and a
non-linear dimensionality reduction technique performed very well.
We do not have a scientific explanation of the reasons behind this success,
but we conjecture that the visual aspect of the items plays a crucial role
which is hard to capture in a small set of visual features. 
Moreover, the small content pool compared to the number of requests served
allows the algorithm to converge quickly also without information about
the similarity between the actions.
To verify the contribution of the visual aspects to customers' decisions and
the best way to encode the visual representation of the images associated
to musical items is left as future work.

\subsection{Position bias estimation}
As reported in the previous section, we tested the Thompson Sampling ranking algorithm
online in combination with the automatic position bias estimation leveraging
expectation maximization (previously called LinTS-PBMRank(EM)).
While we obtained positive results in the online experiment, we observed an unexpected
behaviour in the probabilities computed by the EM algorithm which could have been related
to numerical stability issues and further investigated the matter. 
We decided to run a new online experiment where LinTS-PBMRank(EM) was compared
with an instance of the same algorithm whose position bias probabilities where 
manually tuned leveraging historical data. 
This experiment terminated with a significant victory (about 5\% increase in MRR) 
of the algorithm 
using manually tuned position bias.
Re-applying offline part of the updates to the model, we noticed that even
using a consistent number of updates, in the order of $10^5$, the posteriors means
of the two models where not converging to the same value. Specifically, their cosine
similarity was in the interval (0.6, 0.8).
This is due to two main reasons: i) the random initialization of the EM model and 
ii) the error made by the predictor in estimating the rewards. 
We decided to change the initialization of the EM model to $\frac{1}{\ell+\epsilon}$
where $\epsilon$ is just a small random number (e.g., in (0, 0.1)).
The same offline analysis described above provide significantly different results
using this initialization, with an average cosine similarity of the posterior means at 0.93 
and negligible variance.
We tested a few others initializations techniques offline with slightly worse but
comparable results and we are waiting to validate our findings in online experiments.

\section{Conclusions and Future Work}
\label{sec:conc}
We provided extensions of two well-known contextual bandit algorithms that show a significant empirical advantage in real-world scenarios.
Our online experiments were run on a large scale music streaming service show a significant customer impact
measured by a few different metrics. Moreover, the presented algorithms proved themselves easy to maintain in a production environment.

There are a few directions in which we are considering extending these ranking solutions:
i) perform additional experiment on the most effective representations to be 
used for music recommendations in visual clients, ii) scale known techniques~\cite{gentile2014online,gentile2017context}
for the multi-bandit setting to support a massive number of customers
iii) compare our results with the ones obtained by more complex solutions based on complex
reinforcement learning algorithms.

\section{Acknowledgements}
We would like to thank Claire Vernade for the contributions made
during the inital stage of this project.

\bibliographystyle{plain}
%\bibliography{refs}
\bibliography{Paper}

\newpage
\appendix
% \section*{Appendix}
% \label{sec:appendix}
% \input{Appendix.tex}
% No Appendix so far but just in case

\end{document}